\documentclass[mnsc,nonblindrev]{informs3} 
\OneAndAHalfSpacedXI 

\usepackage{tikz}
\usetikzlibrary{arrows.meta,positioning}
\usepackage{natbib}
\usepackage{dsfont}
\usepackage[ruled,vlined]{algorithm2e}
\DontPrintSemicolon
\usepackage[unicode=true,pdfusetitle, bookmarks=true,bookmarksnumbered=true,bookmarksopen=false, breaklinks=true,pdfborder={0 0 0},pdfborderstyle={},backref=page,colorlinks=true,citecolor=blue]{hyperref} \usetikzlibrary{decorations.pathreplacing,calc}

\global\long\def\E{\mathbb{E}}%
\global\long\def\KL{D_{\mathrm{KL}}}
\global\long\def\Reg{\mathrm{Reg}}
\global\long\def\Unif{\mathrm{Unif}}
\global\long\def\P{\mathbb{P}}
\usepackage{endnotes}
\usepackage{graphicx,hyperref,comment,appendix}
\usepackage[capitalize,noabbrev]{cleveref}
\usepackage{tikz}
\usepackage{subfigure,caption,microtype,booktabs} 

\usepackage{amsthm,amsmath,amssymb,soul}
\usepackage{comment,enumitem}
\setlist{leftmargin=*}

\newtheorem{assumption}{Assumption}

\newtheorem{theorem}{Theorem}

\newtheorem{lemma}{Lemma}

\newtheorem{remark}{Remark}

\usepackage{natbib}
\bibpunct[, ]{(}{)}{,}{a}{}{,}%
\def\bibfont{\small}%
\usepackage{bm,comment}
\usepackage{soul}

\ECRepeatTheorems

\EquationsNumberedThrough    

\allowdisplaybreaks

\begin{document}
\RUNAUTHOR{}

\RUNTITLE{On the Peril of (Even a Little) Nonstationarity in Satisficing Regret Minimization}



\TITLE{On the Peril of (Even a Little) Nonstationarity in Satisficing Regret Minimization}

\ARTICLEAUTHORS{

\AUTHOR{Yixuan Zhang}
\AFF{Department of Industrial \& Systems Engineering, 
University of Wisconsin-Madison\\
	\EMAIL{yzhang2554@wisc.edu}} 

\AUTHOR{Ruihao Zhu}
\AFF{SC Johnson College of Business, Cornell University\\
	\EMAIL{ruihao.zhu@cornell.edu}}

    \AUTHOR{Qiaomin Xie}
\AFF{Department of Industrial \& Systems Engineering, 
University of Wisconsin-Madison\\
	\EMAIL{qiaomin.xie@wisc.edu}} 

} 

\ABSTRACT{%
Motivated by the principle of satisficing in decision-making, we study  
satisficing regret guarantees for nonstationary \(K\)-armed bandits. We show that in the general
realizable, piecewise-stationary setting with \(L\) stationary segments, the optimal regret is \(\Theta(L\log T)\) as long as $L\geq 2$. This stands in sharp contrast to the case of $L=1$ (i.e., the stationary setting), where a $T$-independent \(\Theta(1)\) satisficing regret is achievable under realizability. In other words, the optimal regret has to scale with $T$ even if just a little nonstationarity presents. A key ingredient in our analysis is a novel Fano-based framework tailored to nonstationary bandits via a  \emph{post-interaction reference} construction. This framework strictly extends the classical Fano method for passive estimation as well as recent
interactive Fano techniques for stationary bandits. As a complement, we also discuss a special regime in which constant satisficing regret is again possible.
}%



\maketitle


\section{Introduction}
The multi-armed bandit (MAB) problem is a classical framework for studying the
exploration--exploitation trade-off in sequential decision-making. In the
standard MAB setting, a learner repeatedly chooses among \(K\) arms with unknown,
arm-specific stationary reward distributions, with the goal of minimizing
cumulative regret \citep{robbins1952some,lai1985asymptotically}. For the
canonical MAB problem and its many variants, a rich literature has established
sharp performance guarantees; see, for example, \citet{lattimore2020bandit}.

While identifying the optimal action is essential in some high-stakes settings,
many practical applications only require a ``good-enough,'' or
\emph{satisficing}, decision. Introduced by \citet{simon1956rational},
satisficing aims for an acceptable solution rather than the absolute best one.
This perspective is common in practice. For instance, recommendation and
advertising systems often deploy a candidate only if its KPI exceeds a threshold
\(S\) with sufficient confidence \citep{kohavi2020trustworthy,xu2015infrastructure},
and inventory control typically targets a desired service level (e.g., 90--95\%)
to mitigate stockout costs \citep{fitzsimons2000consumer,jing2011stockouts}.

These considerations motivate the study of \emph{satisficing bandits}
\citep{tamatsukuri2019guaranteed,michel2022regret,rouyer2024understanding,feng2024satisficing,hajiabolhassan2025online},
where the learner aims to pull a \emph{satisficing arm}---an arm whose mean reward
is no less than a prescribed threshold $S$---as often as possible. Performance is
measured by the \emph{satisficing regret}, 
defined as the cumulative
shortfall of the selected arm's mean reward relative to $S$. To the best of our knowledge, existing results on
satisficing regret primarily focus on stationary environments:
\cite{tamatsukuri2019guaranteed,michel2022regret,rouyer2024understanding} study the
$K$-armed bandit setting, \cite{feng2024satisficing} extend the framework to more
general bandit optimization problems, and \cite{hajiabolhassan2025online} study
the average-reward Markov decision process setting. Under a \emph{realizability}
condition—namely, that at least one (initially) unknown arm is strictly satisficing with mean
exceeding $S$—these works show that the satisficing regret can be bounded by a constant
\textit{independent of the decision horizon $T$}. 

In many real-world scenarios, however, stationarity fails: consumer preferences
drift with seasons and trends \citep{ardeshiri2019seasonality}, and platform
marketplaces (e.g., ride-hailing) experience demand surges and changing
competitive conditions \citep{battifarano2019predicting}. As a result, an arm's
reward distribution may evolve over time. Despite the growing literature on
satisficing regret in stationary environments, much less is understood in the
nonstationary setting that frequently arises in practice. Motivated by this gap,
we study satisficing regret minimization under nonstationarity, focusing on
piecewise-stationary bandits with \(L\) stationary segments.
Building on the stationary case (corresponding to \(L=1\)), it is natural to ask:
\begin{center}\emph{Can we similarly achieve a $T$-independent satisficing regret in the piecewise-stationary setting?}\end{center}
At first glance, this seems plausible:   if the change points were known
in advance, one could simply restart a stationary satisficing algorithm on each
segment, leading to an \(\mathcal O(L)\) satisficing regret.



\subsection{Our Contributions}
In this work, however, we provide a negative answer to the above conjecture. Specifically,
\begin{itemize} \item In Section~\ref{sec:general}, we establish an
\(\Omega\bigl(L\log(T/L)\bigr)\) lower bound on the satisficing regret for nonstationary \(K\)-armed bandits. We then propose Algorithm~\ref{alg:1}, which achieves an almost matching upper bound of \(\mathcal{O}(L\log T)\). Together, these results show that the additional $T-$dependent \(\log T\) factor is unavoidable in general when the change points are unknown, thereby ruling out the conjectured \(\mathcal{O}(L)\) satisficing regret guarantee.

\item To establish the lower bounds above, we extend recent \emph{interactive} Fano techniques for bandit regret lower bounds~\citep{chen2024assouad,shufaro2024bits} from stationary to nonstationary environments. More importantly, we introduce a \emph{post-interaction reference} framework that strictly generalizes the classical Fano approach to make it well suited to nonstationary bandits. In this framework, a key distinction compared to the Fano techniques is that part of the (initially) unknown environment parameters (the ``reference'') are revealed to the estimator \emph{after} the entire horizon of length $T$. This leads to two novel complementary  paths in lower bounding the regret (Section~\ref{sec:ourfano1}): (i) reducing a \emph{portion} of the regret to an identification error under a suitable separation condition (used in establishing the lower bound when $L\geq 3$); and (ii) interpreting a \emph{portion} of the regret as an \emph{information cost} (used for \(L=2\)). We also note that this framework may extend to broader classes of nonstationary bandit problems and may be of independent interest.

\item In Section~\ref{sec:specialcase}, we complement the general negative results by identifying a special case in which constant satisficing regret is achievable. Specifically, when there exists an \emph{always-realizable} arm and no satisficing arm ever becomes non-satisficing, following the results in \citet{michel2022regret} can yield a \emph{constant} satisficing regret that is independent of both \(T\) and \(L\).\end{itemize}

\subsection{Related Work}\label{sec:related-work}
\noindent\textbf{Nonstationary Bandits.} 
Most work on nonstationary bandits evaluates performance through
\emph{dynamic regret}, which compares the learner with the best arm at each
round. For piecewise-stationary environments with at most \(L\) segments,
sliding-window and change-point-based methods achieve
\(\widetilde{\mathcal O}(\sqrt{LT})\) dynamic regret
\citep{garivier2011upper,auer2019adaptively}. Under a variation-budget model
with total drift \(D\), the optimal rate is
\(\widetilde{\mathcal O}(D^{1/3}T^{2/3})\)
\citep{besbes2014stochastic,cheung2019learning}, and \citet{wei2021non}
obtain the unified guarantee
\(\widetilde{\mathcal O}(\min\{\sqrt{LT},D^{1/3}T^{2/3}\})\).
These results address a stronger benchmark than ours: dynamic regret competes
with the instantaneous best arm, whereas our objective only requires selecting
an arm whose mean exceeds a fixed threshold.

\noindent\textbf{Stationary Satisficing Regret.} 
Our work is most closely related to the literature on threshold-based
satisficing regret in stationary bandits.
\citet{michel2022regret}, \citet{rouyer2024understanding}, and
\citet{feng2024satisficing} study cumulative satisficing regret under a fixed
threshold and show that, under realizability, constant regret is achievable in
stationary settings. Our paper shows that this phenomenon changes sharply under
nonstationarity: once the environment may change at unknown times, constant
satisficing regret is no longer achievable in general.

\noindent\textbf{Other Satisficing Notions.} 
In stationary settings, threshold-based objectives have also been studied in
pure exploration under the name \emph{thresholding bandits}
\citep{locatelli2016optimal,mukherjee2017thresholding}. There, the goal is to
identify arms above a threshold, with performance measured by error probability
or sample complexity rather than cumulative satisficing regret. Moreover, a
naive explore--then--exploit reduction does not yield optimal cumulative
satisficing regret even in the stationary case
\citep{michel2022regret,feng2024satisficing}.

\citet{russo2018satisficing} introduce a discounted notion of satisficing in
stationary bandits, motivated by time-sensitive learning. Their setting and
benchmark differ from ours: they study discounted regret rather than cumulative
threshold shortfall.

For nonstationary bandits, as highlighted in \cite{liu2023nonstationary,liu2023definition}, the standard notion of dynamic regret can steer algorithm design toward overly aggressive exploration; it can therefore be more meaningful to compete against a \emph{satisficing} benchmark. \cite{min2023information} study nonstationary bandits in a Bayesian model and introduce a notion of satisficing based on competing with a \emph{satisficing action sequence} \(A^\dagger\)~\cite[Sec.~7]{min2023information}. Their ``satisficing regret'' is defined as regret relative to \(A^\dagger\). In contrast, we study a \emph{threshold-based} notion of satisficing regret, defined as the cumulative shortfall from a fixed threshold \(S\). Our formulation does not require that any arm achieves mean exactly \(S\) at each time. Even if one sets \(A^\dagger\) to be a realizable arm sequence, the two regret notions are not directly comparable: regret relative to \(A^\dagger\) can be negative at times when the learner selects an arm with strictly larger mean than \(A_t^\dagger\), whereas our satisficing regret is nonnegative by definition. Finally, \cite{min2023information} establishes asymptotic lower bounds for the classical dynamic regret, while we provide minimax-optimal \emph{finite-time} guarantees for threshold-based satisficing regret.
\subsection{Notation}
For integers \(M\ge 1\), define \([M]:=\{1,2,\ldots,M\}\). For integers \(0\le N\le M\),
define \([N,M]:=\{N,N+1,\ldots,M\}\).
 For a buffer
$\mathcal{B}=(b_1,\ldots,b_m)$ of length $m$, we write $\mathcal{B}[s]:=b_s$ for
its $s$-th entry, $s\in[m]$. For a set $\mathcal A$, we use $|\mathcal A|$ to denote its cardinality, and
$\Unif(\mathcal A)$ for the uniform distribution on $\mathcal A$. For random variables $(X,Y)$ with joint law $P_{X,Y}$, the mutual information is
\begin{equation}\label{eq:mutual}
I(X;Y)
:= \mathbb{E}_{X}\left[
D_{\mathrm{KL}}\bigl(P_{Y\mid X}\,\|\,P_Y\bigr)
\right],
\end{equation}
where $P_{Y\mid X}$ is the conditional law of $Y$ given $X$, $P_Y$ is the
marginal law of $Y$, and $D_{\mathrm{KL}}(\cdot\|\cdot)$ denotes the
Kullback--Leibler (KL) divergence. We also use the conditional mutual
information $I(X;Y\mid Z)
:= \mathbb{E}_{Z}\left[
\mathbb{E}_{X\mid Z}\left[
D_{\mathrm{KL}}\bigl(P_{Y\mid X,Z}\,\|\,P_{Y\mid Z}\bigr)
\right]\right].$
We use the standard Landau notations $o(\cdot)$, $\mathcal{O}(\cdot)$, $\Omega(\cdot)$, $\Theta(\cdot)$ and
$\widetilde{\mathcal{O}}(\cdot)$, as well as the comparison symbols $\lesssim$, $\gtrsim$ and $\asymp$. Unless otherwise stated, these notations hide positive constants that are independent of the horizon \(T\) and the number of stationary segments \(L\) (and \(\widetilde{\mathcal O}(\cdot)\) also suppresses polylogarithmic factors on $T$).

\section{Problem Formulation}
We study a nonstationary \(K\)-armed bandit over a horizon \(T\in\mathbb N\),
with \(K\lesssim T\). At each round \(t\in[T]\), each arm
\(a\in[K]\) has an unknown mean reward \(\mu_t(a)\in[0,1]\).
For integers \(1\le s\le t\le T\), define the number of stationary segments on
\([s,t]\) as
\[
L[s,t]
:= 1+\sum_{j=s}^{t-1}\mathds{1}\left\{\exists a\in[K]  \text{ such that }\ \mu_j(a)\neq \mu_{j+1}(a)\right\}.
\]
Let \(L:=L[1,T]\) denote the total number of stationary segments and $1=T_0<T_1<\cdots<T_L=T+1$
be the associated change-points that partition the horizon into \(L\) stationary
segments, where segment \(\ell\in\{0,\dots,L-1\}\) is the interval
\([T_\ell,\,T_{\ell+1}-1]\). A (possibly randomized) policy \(\pi\) selects an
action \(a_t\in[K]\) at time \(t\), adapted to the natural filtration
\(\mathcal F_{t-1}:=\sigma(a_1,r_1,\dots,a_{t-1},r_{t-1})\). Let \(\Pi\) denote
the class of all such policies. Conditional on \(a_t\), the learner observes a
reward $r_t = \mu_t(a_t) + \xi_t,$ where \((\xi_t)_{t\ge1}\) are independent, zero-mean, \(1\)-sub-Gaussian noises. For each segment \(\ell\in\{0,\dots,L-1\}\), let
\(\mu_{a,\ell}:=\mu_{T_\ell}(a)\) denote the  mean reward of arm
\(a\in[K]\) on segment \(\ell\), and define the segment-wise optimal mean
\(\mu_\ell^* := \max_{a\in[K]} \mu_{a,\ell}\).

The decision maker specifies a threshold \(S\in\mathbb R\), and performance is
judged by whether the chosen arm meets this level. An arm \(a\) is
\emph{satisficing} at time \(t\) if \(\mu_t(a)\ge S\). The one-step satisficing regret is
\((S-\mu_t(a_t))_+ := \max\{0,S-\mu_t(a_t)\}\), and the cumulative
\emph{satisficing regret}  under policy $\pi$ is
\begin{equation}\label{eq:satisficing regret}
R_T^S(\pi;\nu)
:=
\mathbb{E}_{\pi,\nu}\Big[
\sum_{t=1}^T (S-\mu_t(a_t))_+
\Big],
\end{equation}
where \(\nu\) denotes the underlying nonstationary environment  and \(\mathbb E_{\pi,\nu}\) is over the
policy's randomness and the observation noise. Throughout, we focus on the standard realizable case as specified below \citep{michel2022regret,feng2024satisficing}.
\begin{assumption}[Realizability]\label{assumption:fully-satisficing}
For every segment \(\ell\in\{0,\dots,L-1\}\), there exists at least one arm with
mean strictly above the threshold, i.e., \(\mu_\ell^* > S\).
\end{assumption}


\section{A Fano-Based Framework for Nonstationary Bandit Problems
}\label{sec:fano}
In this section, we detail our new technical development based on Fano's lemma~\cite[Theorem~2.10]{reza1994introduction} . We start by reviewing its usage for lower bounds in \emph{stationary} bandits (Section~\ref{sec:ourfano}), and then derive the post-interaction reference framework (Section~\ref{sec:ourfano1}). This yields two powerful approaches that are particularly suitable for deriving
regret lower bounds for \emph{nonstationary} bandits.

Fano's lemma is a classical tool
for establishing minimax lower bounds in multi-hypothesis testing. Let \(\mathcal V\) be a
finite index set, \(\nu\sim\Unif(\mathcal V)\) be a latent instance index,
and \(Z\) be the data generated by an algorithm interacting with the
instance indexed by \(\nu\), Fano's lemma lower bounds the probability of being unable to correctly identify $\nu$ with $Z$ in terms of the mutual information \(I(\nu;Z)\):
\begin{equation}\label{eq:fano}
\inf_{\widehat \nu(Z)} \P\bigl(\widehat \nu(Z) \neq \nu\bigr)
\ge
1-\frac{I(\nu;Z)+\log 2}{\log|\mathcal V|}.
\end{equation}
Being initially designed for multi-hypothesis testing, recent works \citep{chen2024assouad,shufaro2024bits} build on it to develop a general framework for proving regret lower bounds in
\emph{stationary} bandit problems.

\noindent\textbf{Conditional Fano.} Aside from the above, it is often useful to apply Fano's lemma in a structured setting where the index takes the form
\(\nu=(u,b)\in\mathcal U\times\mathcal B\). Here, $b$ is a ``reference'' variable
independent of \(u\). The observation \(Z\) may depend on \(b\), and \(b\) is
revealed to the estimator before it outputs \(\widehat u\).
A conditional form of Fano's inequality~\citep[Theorem~3]{scarlett2019introductory} then
follows directly from~\eqref{eq:fano}:
\begin{equation}\label{eq:fano1}
\inf_{\widehat u(Z,b)} \P\bigl(\widehat u(Z,b) \neq u\bigr)
\ge
1-\frac{I(u;Z\mid b)+\log 2}{\log|\mathcal U|}.
\end{equation}

\subsection{Classical Fano Reduction: Small Regret Implies Correct Identification}
\label{sec:ourfano}
For a bandit problem, we write out \(Z := (a_1,r_1,\dots,a_T,r_T)\) as the full interaction transcript
obtained by running \(\pi\) for \(T\) rounds under the environment indexed by \(\nu\).
Let \(P_\nu\) denote the law of \(Z\) conditional on \(\nu\) and
\(\bar P := \frac{1}{|\mathcal V|}(\sum_{\tau\in\mathcal V} P_\tau)\) be the
resulting mixture distribution.  By definition of mutual information and the
convexity of \(D_{\mathrm{KL}}(P\|\cdot)\) in its second argument
\citep{cover1999elements}, we have
\begin{equation}\label{eq:mutualdecomposition}
I(\nu;Z)
=
\frac{1}{|\mathcal V|}\sum_{\tau\in\mathcal V}
\KL\left(P_\tau \middle\| \bar P\right)
\le
\frac{1}{|\mathcal V|^2}\sum_{\tau,\tau'\in\mathcal V}
\KL\left(P_\tau \middle\| P_{\tau'}\right).
\end{equation}
Combining~\eqref{eq:mutualdecomposition} with~\eqref{eq:fano} yields
\begin{equation}\label{eq:fano-bandit}
\inf_{\widehat \nu(Z)} \P\!\bigl(\widehat \nu(Z) \neq \nu\bigr)
\ge
1-\frac{
\frac{1}{|\mathcal V|^2}\sum_{\tau,\tau'\in\mathcal V}
\KL\left(P_\tau \middle\| P_{\tau'}\right)+\log 2
}{
\log|\mathcal V|
}.
\end{equation}
In stationary bandits, a standard route to regret lower bounds is to
\emph{reduce small regret to accurate identification} of the underlying
instance $\nu$. Specifically, one constructs an instance family
\(\{P_\nu\}_{\nu\in\mathcal V}\) satisfying a \emph{separation condition} (among different arms' mean rewards) at
level \(\Delta\) (see, e.g., \citealt[Eq.~(10)]{chen2024assouad}), which enables the
construction of an estimator \(\widehat \nu(Z)\) from the transcript \(Z\) such
that
\begin{equation}\label{eq:stationary-reduction}
\P\bigl(\widehat \nu(Z)\neq \nu\bigr)
\le
\P\bigl(\Reg_T(\pi,\nu)\ge T\Delta\bigl),
\end{equation}
where \(\Reg_T\) denotes the regret of interest.
By Markov's inequality and~\eqref{eq:fano-bandit}, the regret under the
uniform prior satisfies
\[
\mathbb E\!\left[\Reg_T(\pi,\nu)\right]
\;\ge\;
T\Delta\cdot
\Bigl(
1-\frac{
\frac{1}{|\mathcal V|^2}\sum_{\tau,\tau'\in\mathcal V}
\KL\!\left(P_\tau \middle\| P_{\tau'}\right)+\log 2
}{
\log|\mathcal V|
}
\Bigr),
\]
and one then optimizes over the choice of \(\{P_\nu\}_{\nu\in\mathcal V}\) to establish the lower bound.


\subsection{Post-Interaction Reference and Two Complementary Approaches}
\label{sec:ourfano1}
In nonstationary bandits, however, the reduction~\eqref{eq:stationary-reduction}
typically fails because the canonical separation condition may not hold due to changes in the arms' mean rewards. To overcome this, we introduce the \emph{post-interaction reference} based on the conditional bound~\eqref{eq:fano1}. To the best of our knowledge, our framework is the first systematic application of the conditional Fano's method to establish
lower bounds in (nonstationary) bandits.


\noindent\textbf{Post-Interaction Reference Protocol.}
Draw an environment index \( (u,b) \sim \Unif(\mathcal V)\), where
\(\mathcal V := \mathcal U \times \mathcal B\). Under the prior
\(\Unif(\mathcal V)\), the random variables \(u\) and \(b\) are independent. Under the (unknown) environment \((u,b)\), we let
\(P_{u,b}\) denote the law of \(Z\) conditional on \((u,b)\).
The key feature of the post-interaction reference framework is that \(b\) (the \emph{reference}) is revealed to the
estimator only \emph{after} \(Z\) has been generated. This enables the below two complementary approaches for proving regret lower bounds. 


\noindent\textbf{Approach I: Coordinate-Wise Identification Error.}
Analogous to~\eqref{eq:mutualdecomposition}, conditioning on \(b\) yields
\[
I(u;Z \mid b)
=
\frac{1}{|\mathcal B|}\sum_{x \in \mathcal B}\frac{1}{|\mathcal U|}
\sum_{\tau\in\mathcal U} \KL(P_{\tau,x} \| \bar P_x)
\le
\frac{1}{|\mathcal B||\mathcal U|^2}
\sum_{x \in \mathcal B}\sum_{\tau, \tau'\in\mathcal U}
\KL(P_{\tau,x}\| P_{\tau',x}),
\]
where \(\bar P_x := \frac{1}{|\mathcal U|}\sum_{\tau\in\mathcal U} P_{\tau,x}\).
Plugging this into~\eqref{eq:fano1} gives
\begin{equation}\label{eq:fano1-bandit}
\inf_{\widehat u(Z,b)} \P\bigl(\widehat u(Z,b) \neq u\bigr)
\ge
1-\frac{
\frac{1}{|\mathcal B||\mathcal U|^2}
\sum_{x \in \mathcal B}\sum_{\tau, \tau'\in\mathcal U}
\KL(P_{\tau,x}\| P_{\tau',x})+\log 2
}{
\log|\mathcal U|
}.
\end{equation}
A particularly useful instantiation is when \(\mathcal V\) has a product
structure \(\mathcal V=\mathcal V_1\times\cdots\times\mathcal V_m\). For any
coordinate \(i\in[m]\), we can take
\begin{equation}\label{eq:product}
\mathcal U=\mathcal V_i,
\qquad
\mathcal B=\mathcal V_{-i}
:=\mathcal V_1\times\cdots\times\mathcal V_{i-1}\times
\mathcal V_{i+1}\times\cdots\times\mathcal V_m.
\end{equation}
This coordinate-wise construction is particularly well suited to
nonstationary bandits: each factor \(\mathcal V_i\) can be viewed, roughly, as
an instance family associated with one stationary segment. This structure is
key to obtaining the optimal dependence on the number of stationary segments.
In particular, it yields a coordinate-wise identification bound that can be
paired with a lower bound on a corresponding \emph{partial} regret term:
\begin{equation}\label{eq:fano1-reduction}
\inf_{\widehat u(Z,b)} \P\bigl(\widehat u(Z,b) \neq u\bigr)
\lesssim
\mathbb E\big[\Reg_T^{(i)}(\pi,\nu)\big],
\end{equation}
and the total regret decomposes as
\(\mathbb E[\Reg_T(\pi;\nu)]=\sum_{i=1}^m \mathbb E[\Reg_T^{(i)}(\pi;\nu)]\). 
Throughout this section, the notation $\lesssim$ hides constants that may depend on  $T$ and $\{P_\nu\}_{\nu\in\mathcal V}$.

The intuition for \eqref{eq:fano1-bandit}
and~\eqref{eq:fano1-reduction} is as follows:  for instances that under an appropriate separation
condition on \(\mathcal U\), once the reference $b\in\mathcal{B}$
is revealed, any policy \(\pi\) must incur a nontrivial coordinate-wise
identification error. This lower bound on the identification error can in turn
be converted into a lower bound on the associated partial regret in \eqref{eq:fano1-reduction}. Summing them up and optimizing over
\(\{P_\nu\}_{\nu\in\mathcal V}\) then yields regret lower bounds.
\emph{Approach~I} strictly extends the interactive Fano methods for stationary
bandits outlined in Section~\ref{sec:ourfano}: it allows a broader separation
condition and links identification to \emph{partial} regret via the
post-interaction reference. It underlies our proof of
Theorem~\ref{thm:lower-general}, which establishes the logarithmic lower bound
for \(L \geq 3\).

However, as discussed after Theorem~\ref{thm:lower-general}, when \(L=2\) it
is unclear how to apply \emph{Approach~I}. To overcome this, we develop a
complementary method, again based on the conditional Fano
inequality~\eqref{eq:fano1} and the post-interaction reference.

\noindent\textbf{Approach II: Coordinate-Wise Information Cost.}
We again work under the product structure of \(\mathcal V\), with
\(\mathcal U\) and \(\mathcal B\) chosen as in~\eqref{eq:product} for each
\(i\in[m]\). Conditional on \((\mathcal F_{t-1},b)\), the action \(a_t\) is
drawn from the policy's internal randomness and is therefore independent of
\(u\). Hence \(I(u;a_t\mid \mathcal F_{t-1},b)=0\), and the chain rule gives
\[
I(u;Z\mid b)
=
\sum_{t=1}^T I\bigl(u;a_t\mid \mathcal F_{t-1},b\bigr)
+
\sum_{t=1}^T I\bigl(u;r_t\mid \mathcal F_{t-1},a_t,b\bigr)
=
\sum_{t=1}^T I\bigl(u;r_t\mid \mathcal F_{t-1},a_t,b\bigr).
\]
Consequently,~\eqref{eq:fano1} implies
\begin{equation}\label{eq:fano1-bandit1}
\inf_{\widehat u(Z,b)} \P\!\bigl(\widehat u(Z,b) \neq u\bigr)
\ge
1-\frac{
\sum_{t=1}^T I\bigl(u;r_t\mid \mathcal F_{t-1},a_t,b\bigr)+\log 2
}{
\log|\mathcal U|
}.
\end{equation}
The advantage of~\eqref{eq:fano1-bandit1} is that it yields regret lower bounds
\emph{without} requiring a direct reduction from small (partial) regret to
correct (coordinate-wise) identification, a step needed in Approach~I but one
that may fail in some nonstationary constructions. Instead, we take an
\emph{information-cost} view: we construct \(\{P_\nu\}_{\nu\in\mathcal V}\) and
an estimator \(\widehat u(Z,b)\) such that for some \(\rho\in(0,1)\),
\begin{equation}\label{eq:fano1-reduction1}
\underbrace{
\P\bigl(\widehat u(Z,b) \neq u\bigr)\le \rho
,}_{\text{Identification error given the reference}}
\qquad \underbrace{
\sum_{t=1}^T I\bigl(u; r_t \mid \mathcal F_{t-1}, a_t, b\bigr)
\lesssim
\mathbb E\big[\Reg_T^{(i)}(\pi,\nu)\big]
}_{\text{Coordinate-wise cost of information paid by partial regret}}
.
\end{equation}
An intuitive interpretation of \eqref{eq:fano1-bandit1} and \eqref{eq:fano1-reduction1} is as follows: once the reference is revealed,
any policy \(\pi\) must achieve a nontrivial level of coordinate-wise
identification accuracy. Thus, \eqref{eq:fano1-bandit1} yields a lower bound on
the information that must be acquired. In our construction, acquiring this
information necessarily incurs partial regret. We use this approach to prove
Theorem~\ref{thm:lower-general-1}, which establishes the logarithmic lower
bound for \(L=2\) and complements Theorem~\ref{thm:lower-general}.
\begin{remark}[Comparison with the Classical KL-Based Approach]
Classical KL-based arguments also relate regret to information acquisition, but
they are inherently pairwise: they compare transcript laws under two
alternative environments and convert the resulting KL divergence into a regret
lower bound \citep{bubeck2012regret,garivier2019explore,garivier2011upper}. By
contrast, our construction is genuinely \emph{multi-hypothesis}, via Fano's
inequality, and is therefore better suited to capture the optimal dependence
on the number of hypotheses and hence sharper regret lower bounds. Moreover,
the relevant information measure in our framework is (conditional) mutual
information, which is central to our analysis and can sometimes be linked to
regret more directly than KL divergence. We provide a more detailed discussion of this after 
Theorem~\ref{thm:lower-general-1}.
\end{remark}

\section{Regret Guarantees for Nonstationary Satisficing Bandits}
\label{sec:general}

In this section, we establish optimal satisficing regret guarantees for general nonstationary satisficing bandits. Sections~\ref{sec:general-lower} and \ref{sec:general-upper} present the corresponding lower and upper bounds, respectively.

\subsection{Lower Bounds}
\label{sec:general-lower}

In the stationary setting, optimal satisficing regret is $T$-independent under
realizability~\citep{feng2024satisficing,rouyer2024understanding}.
However, as we show below, such a rate is generally impossible when change points are unknown in advance.

\noindent\textbf{The Hard Instances.} Consider a fixed $\Delta>0$, and let \(\mathcal{E}_{L,\Delta}\) denote the class of two-armed,
piecewise-stationary bandit instances with \(L\) segments, where rewards are
Gaussian with unit variance and the mean vector on each segment \(\ell\) is
either \((S+\Delta,S-\Delta)\) or \((S-\Delta,S+\Delta)\), where we recall $S$ is the satisficing threshold. In particular, the identity of the
unique satisficing arm changes at every change point. Using the post-interaction reference
framework of Section~\ref{sec:ourfano1}, we derive a sharp satisficing regret
lower bound on \(\mathcal E_{L,\Delta}\).
\begin{theorem}
\label{thm:lower-general}
Suppose \(L\ge 3\) and \(\Delta^2 T \ge L\). Then $\inf_{\pi \in \Pi}\sup_{\nu \in \mathcal{E}_{L,\Delta}}R_T^S(\pi;\nu)\gtrsim L\log (\Delta^2T/L)/\Delta.$
\end{theorem}
When $\Delta$ is a constant and $L \ll \Delta^2 T$, Theorem~\ref{thm:lower-general} implies an  
\(\Omega(L\log T)\) lower bound for general nonstationary satisficing bandits.
Theorem~\ref{thm:lower-general} reflects the difficulty incurred by unknown
change points: minimizing satisficing regret is at least as hard as localizing
them, which already incurs a \(T\)-dependent \(\Omega(L\log T)\) cost. The
proof, deferred to Section~\ref{sec:outline-lower-general}, is based on
\emph{Approach~I} introduced in Section~\ref{sec:ourfano1}.

As will become clear in the proof, Theorem~\ref{thm:lower-general} requires
\(L \ge 3\) because the construction~\eqref{eq:construction} relies on a short
stationary segment  \([t_1,t_2]\subseteq[2,T-1]\) with $t_2-t_1 \equiv c$. After the reference is revealed, one can
choose two instances that differ only on two short interior intervals, whose
total length $2c$ is independent of their locations. This yields the separation
condition~\eqref{eq:seperation} and hence the reduction~\eqref{eq:fano1-reduction}.
When \(L=2\), however, there is only one change point. If two instances place
it at different times \(\tau\) and \(\tau'\), then they disagree on the entire
interval between \(\tau\) and \(\tau'\), whose length grows with
\(|\tau-\tau'|\) (Figure~\ref{fig:separation}). It is therefore unclear how
to construct a comparable separated instance family in this case and the possibility of $T$-independent satisficing regret remains.

The next theorem resolves this issue via \emph{Approach~II}, introduced in
Section~\ref{sec:ourfano1}, and establishes the same scaling for all
\(L \ge 2\) under a slightly stronger condition relating \(\Delta\), \(T\),
and \(L\).

\begin{figure}[!htbp]
\centering
\begin{tikzpicture}[
  x=1cm,y=1cm,
  line/.style={line width=0.9pt},
  tick/.style={line width=0.9pt},
  brace/.style={decorate,decoration={brace,amplitude=5pt}},
  lab/.style={font=\small}
]

\def\w{7.2}
\def\gap{0.8}
\def\xL{0}
\def\xR{\w}

\def\yTop{2.6}
\def\yBot{1.6}

\begin{scope}[shift={(0,0)}]

\node[lab, anchor=east] at (-0.2,2.6) {Instance 1};
\node[lab, anchor=east] at (-0.2,1.6) {Instance 2};
\node[font=\small\bfseries] at (3.5,3.2) {$L \geq 3$: Short Disagreement};

\fill[gray!25] (2,\yBot) rectangle (2.8,\yTop);
\fill[gray!25] (5,\yBot) rectangle (5.8,\yTop);

\draw[line, blue] (\xL,\yTop) -- (2,\yTop);
\draw[line, red]  (2,\yTop) -- (2.8,\yTop);
\draw[line, blue] (2.8,\yTop) -- (\xR,\yTop);
\draw[tick] (2,\yTop-0.12) -- (2,\yTop+0.12);
\draw[tick] (2.8,\yTop-0.12) -- (2.8,\yTop+0.12);

\draw[line, blue] (\xL,\yBot) -- (5,\yBot);
\draw[line, red]  (5,\yBot) -- (5.8,\yBot);
\draw[line, blue] (5.8,\yBot) -- (\xR,\yBot);
\draw[tick] (5,\yBot-0.12) -- (5,\yBot+0.12);
\draw[tick] (5.8,\yBot-0.12) -- (5.8,\yBot+0.12);

\draw[dashed] (2,\yBot) -- (2,\yTop);
\draw[dashed] (2.8,\yBot) -- (2.8,\yTop);
\draw[dashed] (5,\yBot) -- (5,\yTop);
\draw[dashed] (5.8,\yBot) -- (5.8,\yTop);

\end{scope}

\begin{scope}[shift={(\w+\gap,0)}]

\node[font=\small\bfseries] at (3.5,3.2) {$L = 2$:  Wide Disagreement};

\fill[gray!25] (1,\yBot) rectangle (5.2,\yTop);

\draw[line, blue] (\xL,\yTop) -- (1,\yTop);
\draw[line, red]  (1,\yTop) -- (\xR,\yTop);
\draw[tick] (1,\yTop-0.12) -- (1,\yTop+0.12);

\draw[line, blue] (\xL,\yBot) -- (5.2,\yBot);
\draw[line, red]  (5.2,\yBot) -- (\xR,\yBot);
\draw[tick] (5.2,\yBot-0.12) -- (5.2,\yBot+0.12);

\draw[dashed] (1,\yBot) -- (1,\yTop);
\draw[dashed] (5.2,\yBot) -- (5.2,\yTop);

\end{scope}

\end{tikzpicture}
\caption{Schematic illustration: the full horizontal line represents the time horizon \(T\), blue/red segments denote the two stationary environments, and the shaded region marks where the two instances disagree.}
\label{fig:separation}
\end{figure}

\begin{theorem}
\label{thm:lower-general-1}
Suppose \(L\ge 2\) and \(\Delta^2 T \ge 13L\). Then $\inf_{\pi \in \Pi}\sup_{\nu \in \mathcal{E}_{L,\Delta}}R_T^S(\pi;\nu)\gtrsim L\log (\Delta^2T/L)/\Delta.$
\end{theorem}
The full proof of Theorem~\ref{thm:lower-general-1} is deferred to
Section~\ref{sec:proof-lower-general-1}. The theorem also offers a
complementary intuition for the \(\Omega(L\log T)\) scaling: regardless of the
policy, one must acquire \(\Omega(\log T)\) information per segment, since the
reference can be used to reduce the identification error below a constant. The
information-cost bound then implies that acquiring this information
necessarily incurs \(\Omega(\log T)\) satisficing regret per segment.

\noindent\textbf{Challenge with KL-based Arguments.} We also note that it is unclear whether such results can be derived using a 
KL-based argument. To see the difficulty, consider two
instances \(\nu,\nu'\) that differ only on a disagreement region
\(D(\nu,\nu') \subseteq [T]\). Under Gaussian rewards, the per-round KL
divergence at time \(t\in D(\nu,\nu')\) is of order
\[
\KL\!\bigl(\mathcal N(\mu_t^\nu(a),1),\,\mathcal N(\mu_t^{\nu'}(a),1)\bigr)
\asymp
\bigl(\mu_t^\nu(a)-\mu_t^{\nu'}(a)\bigr)^2
\asymp \Delta^2,
\]
so the total transcript KL satisfies $\KL(P_\nu^\pi,P_{\nu'}^\pi)\approx
\Delta^2\,|D(\nu,\nu')|.$ Thus, the KL divergence is driven primarily by the \emph{length} of the
disagreement region. In particular, it does not directly distinguish whether
the samples collected there incur  satisficing regret or merely provide
information.

Our framework overcomes this difficulty by working instead with
\emph{(conditional) mutual information}, which can be upper bounded by the
satisficing regret itself, so that any nontrivial identification task necessarily incurs satisficing regret.

\subsection{Upper Bounds}\label{sec:general-upper}
We now present an algorithm whose satisficing regret matches the lower bounds (up to problem-dependent constants). We begin by defining the windowed
statistics used throughout.

\noindent\textbf{Windowed Statistics.}
For any  buffer $\mathcal B$ of length
$m$, define the dyadic window set $\mathrm{WinSet}(\mathcal B)
:=\{2^k:2^k\le m,\ k\in\mathbb N\}$. For each $w\in\mathrm{WinSet}(\mathcal B)$, define the average of the most
recent $w$ entries in the buffer $\mathcal B$ by $\mathrm{AvgLast}(\mathcal B,w)
:=\frac{1}{w}\sum_{s=m-w+1}^{m}\mathcal B[s].$ Given a confidence radius $\beta(\cdot)$, define the windowed lower and upper
confidence bounds $\mathrm{LCB}^{\operatorname{win}}(\mathcal B)
:=\max_{w\in\mathrm{WinSet}(\mathcal B)}
\big\{\mathrm{AvgLast}(\mathcal B,w)-\beta(w)\big\}$ and $\mathrm{UCB}^{\operatorname{win}}(\mathcal B)
=\min_{w\in\mathrm{WinSet}(\mathcal B)}
\big\{\mathrm{AvgLast}(\mathcal B,w)+\beta(w)\big\}.$

\begin{algorithm}[!htbp]
\caption{\textsc{Nonstationary--Sat} (algorithm for nonstationary satisficing bandits)}
\label{alg:1}
\textbf{Input:} horizon $T$, threshold $S$, arms $[K]$, $\displaystyle \beta(w)\ \gets\ \sqrt{2\big(4 \log T+\log K\big)/w}$ for any $w\ge1.$

\textbf{State:} \textsf{leader}$\gets$\texttt{None}, set epoch  $\kappa_a\gets 1$ for each $a\in[K]$, $\mathcal B_a^{(\kappa_a)}\gets[]$.

\For{$t=1,2,\dots,T$}{
  \If{\textsf{leader}=\texttt{None}}{
     Pull $a_t$ as the next arm in a fixed round–robin over $[K]$, observe $r_t$, append to $\mathcal B_{a_t}^{(\kappa_{a_t})}$.

     \textbf{if} $\mathrm{LCB}^{\operatorname{win}}(\mathcal B_{a_t}^{(\kappa_{a_t})})\geq S$ \textbf{then} \textsf{leader}$\gets a_t$, $\kappa_{a_t}\gets \kappa_{a_t}+1$, $\mathcal B_{a_t}^{(\kappa_{a_t})}\gets[ ]$.
      \textbf{end}
  }
  \Else{
     Pull $a_t \gets$ \textsf{leader}, observe $r_t$, append to $\mathcal B_{a_t}^{(\kappa_{a_t})}$.

     \textbf{if} $\mathrm{UCB}^{\operatorname{win}}\big(\mathcal B_{a_t}^{(\kappa_{a_t})}\big)< S$ \textbf{then} \textsf{leader}$\gets$\texttt{None},  
        $\kappa_{a_t}\gets \kappa_{a_t}+1$, $\mathcal B_{a_t}^{(\kappa_{a_t})}\gets[ ]$. \textbf{end}
  }
}
\end{algorithm}

Algorithm~\ref{alg:1} maintains a $\textsf{leader}$ state. When
\(\textsf{leader}=\texttt{None}\), it \emph{explores} until it certifies that an
arm is satisficing. When a leader is present, the algorithm \emph{exploits} this
arm while continuously \emph{testing} whether it remains satisficing. The key
difficulty in piecewise-stationary environments is that, after a change point,
stale samples can be misleading. The windowed statistics address this issue:
$\mathrm{LCB}^{\operatorname{win}}(\mathcal B)$ takes the most optimistic lower bound over
dyadic suffix windows, while $\mathrm{UCB}^{\operatorname{win}}(\mathcal B)$
takes the most
pessimistic upper bound. Consequently, if an arm's current mean exceeds \(S\),
then eventually 
$\mathrm{LCB}^{\operatorname{win}}(\mathcal B)\ge S$; conversely, if the
current leader becomes non-satisficing, then for a sufficiently recent window we
have 
$\mathrm{UCB}^{\operatorname{win}}(\mathcal B)<S$, triggering a reset and renewed
exploration.

We now state the satisficing regret guarantee for Algorithm~\ref{alg:1}.

\begin{theorem}\label{thm:upper-general} Let $\pi$ be Algorithm \ref{alg:1}. Under Assumption \ref{assumption:fully-satisficing}, we have
\[
R^S_T(\pi;\nu) \lesssim \sum_{\ell=0}^{L-1}\sum_{a:\mu_{a,\ell}<S}\big(\frac{1}{\Delta_{a,\ell}} + \frac{\Delta_{a,\ell}}{(\Delta_{\ell}^{*})^2}\big)\cdot \log T 
\]
where $\Delta_{a,\ell} := (S-\mu_{a,\ell})_+$ and $
\Delta_{\ell}^* := \mu_\ell^* - S.$
\end{theorem}
Specializing to \(\nu\in\mathcal{E}_{L,\Delta}\),
Theorem~\ref{thm:upper-general} yields an \(\mathcal O(L\log T/\Delta)\) upper
bound, matching Theorems~\ref{thm:lower-general} and \ref{thm:lower-general-1} up to
\(T\)-independent factors. Moreover, the dependence
\(\frac{1}{\Delta_{a,\ell}}+\frac{\Delta_{a,\ell}}{(\Delta_\ell^*)^2}\) also
appears in stationary satisficing regret bounds; see, e.g.,
\cite[Theorem~2]{michel2022regret}.

The $\log T$ factor comes from concentration. With
\(\beta(w)\asymp\sqrt{\log T/w}\), certifying \(\mu>S\) (via 
$\mathrm{LCB}^{\operatorname{win}}\ge S$) or rejecting $\mu<S$ (via
$\mathrm{UCB}^{\operatorname{win}}<S$) requires only \(w=\mathcal O(\log T/\Delta^2)\) samples,
where \(\Delta\) is the relevant gap to the threshold. The windowed bounds also
prevent frequent resets: after a reset, if the environment remains unchanged,
then with high probability a satisficing leader is not demoted and a
non-satisficing candidate is not promoted. Thus, within each stationary segment, the algorithm spends
only logarithmically many rounds either finding a satisficing leader or
detecting that the current leader has fallen below $S$; afterward it plays a
satisficing arm and incurs no further regret. The full proof of Theorem~\ref{thm:upper-general} is deferred
to Section~\ref{sec:proof-upper-general}.

\section{When Is $T-$Independent Satisficing Regret Achievable
}\label{sec:specialcase}
It remains unclear which structural conditions permit \emph{$T-$independent}
satisficing regret in nonstationary \(K\)-armed bandits. A common feature of the hard instances behind the \(\Omega(\log T)\) lower bounds
in Section~\ref{sec:general} is the presence of \emph{threshold crossings}: an
arm may be satisficing early on but later \emph{down-cross} below \(S\) at an
unknown time, and the identity of the realizable  arm may
also switch across segments. Such changes force the learner to repeatedly
certify satisficing arms, which incurs a logarithmic cost.

In this section, we identify a regime in which these obstacles disappear.
Specifically, if (i) there exists an \emph{always-realizable} arm that remains
strictly satisficing throughout, and (ii) no arm ever transitions from
satisficing to non-satisficing (i.e., no down-crossings), then constant
satisficing regret—independent of both the horizon \(T\) and the number of
stationary segments \(L\)—is achievable.

Formally, we
impose the following assumptions.

\begin{assumption}[Always-Realizability]
\label{ass:always-realizable}
$\exists a^*\in[K]$ such that $\mu_{a^*,\ell}>S$ for all $\ell\in[0,L-1]$.
\end{assumption}

\begin{assumption}[No Down-Crossing]\label{ass:nodown}
For each arm \(a\in[K]\), the indicator sequence
\(\bigl\{\mathds{1}\{\mu_t(a)\ge S\}\bigr\}_{t\in[T]}\) is nondecreasing in \(t\).
\end{assumption}
Under Assumptions~\ref{ass:always-realizable} and~\ref{ass:nodown},
\textsc{Simple-Sat} (Algorithm~1 in \cite{michel2022regret}), originally designed
for the stationary setting, can be applied directly in the nonstationary case
and achieves constant satisficing regret. Below, \(\widehat{\mu}_a(t)\) denotes
the empirical mean reward of arm \(a\) based on observations collected up to time
\(t\).

\begin{algorithm}[!htbp]
\caption{\textsc{Simple-Sat}~\citep{michel2022regret}}
\label{alg:3}
\textbf{Input:} threshold $S$, arms $[K]$.

For time steps $t = 1,\dots , K$, pull  arm $a_t = t$.

\For{$t=K+1,K+2,\dots$}{
If there exists an arm $a\in[K]$ with $\widehat{\mu}_a(t)\ge S$,
play an arm in $\argmax_{a\in[K]}\widehat{\mu}_a(t)$; otherwise, choose $a_t$
uniformly at random from $[K]$.
}
\end{algorithm}

\begin{theorem}\label{thm:upper-nodown}
Let
$\pi$ be Algorithm~\ref{alg:3}. Under Assumptions~\ref{ass:always-realizable} and~\ref{ass:nodown}, we  have
\begin{align*}
R_T^S(\pi;\nu)
\lesssim \sum_{a:\Delta_{a,\min}>0}\Delta_{a,\max}\Big(1+\frac{1}{\Delta_{a,\min}^2} + \frac{1}{\Delta_*^2}\Big),
\end{align*}
where $\Delta_{a,\min}:=   \min_{\ell\in[0,L-1]} \{\Delta_{a,\ell}:\ \Delta_{a,\ell}>0\}$, $\Delta_{a,\max}:=   \max_{\ell\in[0,L-1]} \{\Delta_{a,\ell}:\ \Delta_{a,\ell}>0\}$ and $\Delta_* := \min_{\ell\in[0,L-1]} (\mu_{a^*,\ell}-S).$
\end{theorem}
Theorem~\ref{thm:upper-nodown} establishes a constant satisficing regret bound
that does not depend on \(T\) or \(L\), and it recovers the stationary guarantee
when \(L=1\); see \cite{michel2022regret}. The key proof idea is that, under
Assumptions~\ref{ass:always-realizable} and~\ref{ass:nodown}, the always-realizable
arm and any non-satisficing arm can be identified in finite expected time
\emph{before} the latter ever becomes satisficing. The full proof is deferred to
Section~\ref{sec:proof-upper-nodown}. We emphasize that Assumptions~\ref{ass:always-realizable} and~\ref{ass:nodown}
are sufficient (but not necessary) for constant satisficing regret. Characterizing
necessary and sufficient structural conditions remains an interesting direction
for future work.

\section{Proofs}


\subsection{Proof of Theorem \ref{thm:lower-general}}\label{sec:outline-lower-general}
We prove Theorem~\ref{thm:lower-general} using Approach~I in
Section~\ref{sec:ourfano1}.

\noindent\textbf{Construct $\{P_\nu\}_{\nu \in \mathcal V}$.} 
Assume \(\xi_t \stackrel{\mathrm{i.i.d.}}{\sim} \mathcal N(0,1)\). For
simplicity, assume that \(L\) is odd and that \(2T/(L-1)\) is an integer; the
even-\(L\) case can be handled similarly by adding an initial change point at
time \(2\). Let $m:=\frac{2T}{L-1}$ and $\mathcal I_\ell:=\{(\ell-1)m+1,\dots,\ell m\}$ for any $\ell\in[(L-1)/2],$ so the horizon is partitioned into \((L-1)/2\) disjoint blocks. In each block
\(\ell\), there is exactly one special window of length \(l\), whose start time
is unknown. Fix \(n,l\ge 1\) with \(nl\le m-2\), and define the candidate
within-block start times by $\mathcal T:=\{2,2+l,2+2l,\dots,2+(n-1)l\},$ so that the windows \([\tau,\tau+l-1]\) are disjoint for distinct
\(\tau\in\mathcal T\). Draw
\[
\nu=(\nu_1,\dots,\nu_{(L-1)/2}),
\qquad
\nu_\ell \stackrel{\mathrm{i.i.d.}}{\sim} \Unif(\mathcal T),
\]
and interpret \(\nu_\ell\) as the start time of the special window in block
\(\ell\), with corresponding absolute start time $\widetilde{\nu}_\ell := (\ell-1)m+\nu_\ell \in \mathcal I_\ell.$ For each block \(\ell\) and each \(t\in\mathcal I_\ell\), set
\begin{equation}\label{eq:construction}
(\mu_t(1),\mu_t(2))
=
\begin{cases}
(S-\Delta, S+\Delta), & t\in [\widetilde\nu_\ell,\widetilde\nu_\ell+l-1],\\
(S+\Delta, S-\Delta), & t\in \mathcal I_\ell\setminus[\widetilde\nu_\ell,\widetilde\nu_\ell+l-1].
\end{cases}
\end{equation}
This construction yields $L$ stationary segments. Let $Z:=(a_1,r_1,\dots,a_T,r_T)$ be the complete interaction transcript, and let $P_\nu$ denote the law of $Z$ conditional on $\nu$. We write $\nu_{-\ell}:=(\nu_1, \dots, \nu_{\ell-1}, \nu_{\ell+1}, \dots, \nu_{(L-1)/2})$ for the collection of all components except the $\ell$-th. A direct calculation
shows that for any $\ell\in[(L-1)/2]$, any fixed $\tau_{-\ell}$,
and any two distinct  $\tau_\ell\neq \tau_\ell'$, we have
\begin{equation}\label{eq:seperation}\KL(P_{\tau_\ell, \tau_{-\ell}} \| P_{\tau'_\ell, \tau_{-\ell}}) = 4\Delta^2l.
\end{equation}Therefore, by \eqref{eq:fano1-bandit},
\begin{align*}
\inf_{\widehat \nu_\ell(Z,\nu_{-\ell})} \P\bigl(\widehat \nu_\ell(Z,\nu_{-\ell}) \neq \nu_\ell\bigr)
&\ge 1-\frac{\frac{1}{n^{(L-3)/2}n^2}\sum_{\tau_{-\ell}}\sum_{\tau_\ell,\tau'_\ell}
\KL(P_{\tau_\ell, \tau_{-\ell}} \| P_{\tau'_\ell, \tau_{-\ell}})+\log 2}{\log n}\\
&\geq 1-\frac{4\Delta^2l+\log 2}{\log n} \geq \frac{1}{2},
\end{align*}
where the last inequality holds by choosing
$l=\big\lceil\frac{\log n}{16\Delta^2}\big\rceil$ and taking $n$ sufficiently
large.

\noindent\textbf{Construct $\widehat \nu_\ell(Z,\nu_{-\ell})$ to lower bound partial regret.} For each  window $\tau\in\mathcal T$ in block $\ell$, construct the counter and the estimator: 
\[
N_2^{(\ell)}(\tau)
:=\sum_{t=(\ell-1)m+\tau}^{(\ell-1)m+\tau+l-1}\mathds 1\{a_t=2\}, \qquad\widehat \nu_\ell(Z,\nu_{-\ell}):=\arg\max_{\tau\in\mathcal T} N_2^{(\ell)}(\tau), 
\] 
with ties broken arbitrarily. Define the blockwise number of non-satisficing pulls: $M_\ell := \sum_{t\in\mathcal I_\ell}\mathds 1\{a_t\neq a_t^*\},$ where $a_t^*$ is the unique satisficing arm at time $t$.
Noting that inside the true  window arm $2$ is satisficing for
all $l$ rounds, if $M_\ell\le \frac14 l$,  then $N_2^{(\ell)}(\nu_\ell)\ \ge\ l-M_\ell\ \ge\ \frac34 l.$ Consequently, $M_\ell\le \frac14 l$ implies $\widehat \nu_\ell(Z,\nu_{-\ell})=\nu_\ell$, and thus
 \begin{align*}
    \frac{1}{2} \leq  \P\bigl(\widehat \nu_\ell(Z,\nu_{-\ell}) \neq \nu_\ell\bigr)
   \le
  \mathbb{P} \bigl(  M_\ell>\tfrac14 l\bigr)&\leq \frac{4\E[M_\ell]}{l}\le 
  \frac{64\Delta^2 \E[M_\ell]}{\log n}.
 \end{align*}
Then, $\E[M_\ell] \geq \frac{\log n}{128\Delta^2}$ for every $\ell \in [(L-1)/2]$. Since each pull of a non-satisficing arm incurs satisficing regret $\Delta$, the partial regret for block $\ell$ is given by $\Delta\E[M_\ell].$ We thus have 
$$\E[R_T^S(\pi, \nu)]=
\sum_{\ell \in [(L-1)/2]}\Delta\E[M_\ell]\ge \frac{(L-1)\log n}{256\Delta}.$$ 

\noindent\textbf{Optimize $n$.} Recall that $nl+2\le 2T/(L-1)$ and $l=\big\lceil \frac{\log n}{16\Delta^2}\big\rceil$.
Ignoring the ceiling, it suffices to choose $n \geq 2$ such that $n\log n \le 16\Delta^2(\frac{2T}{L-1}-2).$ Since $\Delta\le \tfrac12$ and $T\Delta^2\geq L$, we further have $16\Delta^2(\frac{2T}{L-1}-2)> 4.$ The proof of Theorem~\ref{thm:lower-general} is therefore completed by applying Lemma~\ref{lem:W}.

\begin{lemma}\label{lem:W}
Let $y\ge 4$. Then there exists an integer $x\ge 2$ such that   $x\log x \le y$ and $\log x \ge \frac12 \log y.$
\end{lemma}
By Lemma~\ref{lem:W}, we finish the proof by choosing $n$ such that $\E[R_T^S(\pi, \nu)]
  \gtrsim \frac{L\log (\Delta^2T/L)}{\Delta}.$
\begin{proof}{Proof of Lemma \ref{lem:W}}
Let $x := \lceil \sqrt{y}\rceil$. Then $x\ge \sqrt{y}$, so
$\log x \ge \frac12\log y$. It remains to show $x\log x \le y$. Let $n:=\lfloor \sqrt{y}\rfloor$, so
$n\ge 2$ and $y\in[n^2,(n+1)^2)$, which implies $x=\lceil\sqrt{y}\rceil\leq n+1$.
Thus it suffices to prove 
\begin{equation}\label{eq:key-ineq}
(n+1)\log(n+1) \le n^2 \qquad \text{for all integers } n\ge 2.
\end{equation}
For $n=2$, \eqref{eq:key-ineq} holds since $3\log 3 \le 4$.
For $n\ge 3$, define the function $f(u):=u-1-\log(u+1)$ for $u\ge 3$. Then $f'(u)=1-\frac{1}{u+1}=\frac{u}{u+1}>0,$ so $f$ is increasing on $[3,\infty)$, and
$f(n)\ge f(3)=2-\log 4>0$. Hence $\log(n+1)\le n-1$ for all integers $n\ge 3$,
and therefore $(n+1)\log(n+1) \le (n+1)(n-1) = n^2-1 \le n^2,$ proving \eqref{eq:key-ineq}. Since $y\ge n^2$, we conclude
$x\log x \leq  (n+1)\log(n+1) \le n^2 \le y$. This completes the proof of Lemma \ref{lem:W}.
\end{proof}

\subsection{Proof of Theorem \ref{thm:lower-general-1}}\label{sec:proof-lower-general-1}
We prove Theorem~\ref{thm:lower-general-1} using Approach~II in
Section~\ref{sec:ourfano1}.

\noindent\textbf{Construct $\{P_\nu\}_{\nu \in \mathcal V}$.} Assume \(\xi_t \stackrel{\mathrm{i.i.d.}}{\sim} \mathcal N(0,1)\). For simplicity,
suppose \(L\) is even and \(2T/L\in\mathbb N\). Let $m:=\frac{2T}{L}$ and $\mathcal I_\ell:=\{(\ell-1)m+1,\dots,\ell m\}$ for any $ \ell\in[L/2].$ Then the horizon is partitioned into \(L/2\) disjoint blocks of length \(m\),
each containing one change point. Fix integers \(n,l\ge 1\) with \(nl\le m-1\),
and define $\mathcal T:=\{2,\,2+l,\,\dots,\,2+(n-1)l\},$ so that \([\tau,\tau+l-1]\) and \([\tau',\tau'+l-1]\) are disjoint for
\(\tau\neq\tau'\). Draw
\[
\nu=(\nu_1,\dots,\nu_{L/2}),
\qquad
\nu_\ell \stackrel{\mathrm{i.i.d.}}{\sim} \Unif(\mathcal T),
\]
and interpret \(\nu_\ell\) as the within-block change point in block \(\ell\).
Its absolute location is $\widetilde{\nu}_\ell := (\ell-1)m+\nu_\ell \in \mathcal I_\ell.
$ For each block \(\ell\) and each \(t\in\mathcal I_\ell\), set
\begin{equation}\label{eq:construction1}
(\mu_t(1),\mu_t(2))
=
\begin{cases}
(S-\Delta, S+\Delta), & t\in [(\ell-1)m+1,\widetilde\nu_\ell-1],\\
(S+\Delta, S-\Delta), & t\in [\widetilde\nu_\ell,\ell m].
\end{cases}
\end{equation}
This yields $L$ stationary segments. Let $Z:=(a_1,r_1,\dots,a_T,r_T)$, and let $P_\nu$ denote the law of $Z$ conditional on $\nu$. We write $\nu_{-\ell}:=(\nu_1, \dots, \nu_{\ell-1}, \nu_{\ell+1}, \dots, \nu_{L/2})$ for the collection of all components except the $\ell$-th. Fix \(\ell\in[L/2]\) and \(t\in[T]\), and condition on
\((\mathcal F_{t-1},a_t, \nu_{-\ell})=(f,a,x)\). Define
\[
p := \mathbb{P}(\widetilde \nu_\ell \le t \mid f,a,x)\mathds 1\{a=2\} + \mathbb{P}(\widetilde \nu_\ell > t \mid f,a,x)\mathds 1\{a=1\},
\quad
P_+ := \mathcal N(S+\Delta,1),\quad P_- := \mathcal N(S-\Delta,1).
\]
Given $(f,a,x)$, for any \(t\in\mathcal I_\ell\), the conditional law of $r_t$ under the prior on $\nu_\ell$ is the
mixture $Q := (1-p)P_+ + p P_-$. Moreover, conditionally on $(f,a,x)$, $r_t$ depends on $\nu_\ell$ only through the indicator
$\mathds 1\{\widetilde\nu_\ell\le t\}$. Thus, for any \(t\in\mathcal I_\ell\),
\begin{align*}
I(\nu_\ell;r_t \mid \mathcal F_{t-1}=f, a_t=a, \nu_{-\ell}=x)
&=
I\bigl(\mathds 1\{\widetilde\nu_\ell\le t\}; r_t \mid \mathcal F_{t-1}=f, a_t=a, \nu_{-\ell} = x\bigr)\\
&=
(1-p)\KL\bigl(P_+ \| Q\bigr)
+ p\KL\bigl(P_- \| Q\bigr).
\end{align*}
If \(t\notin \mathcal I_\ell\), then \(r_t\) does not depend on \(\nu_\ell\) given
\((\mathcal F_{t-1},a_t,\nu_{-\ell})\), and therefore
\begin{align*}
I(\nu_\ell;r_t \mid \mathcal F_{t-1}=f, a_t=a, \nu_{-\ell}=x) = 0, \quad \forall t \in [T] \setminus \mathcal I_{\ell}.
\end{align*}
By convexity of $\KL$ in its second argument, $\KL(P_+ \| Q)
\le p\KL(P_+ \| P_-)$ and $\KL(P_- \| Q)
\le (1-p)\KL(P_- \| P_+).$ Thus, for any \(t\in\mathcal I_\ell\),
\begin{align*}
I(\nu_\ell;r_t \mid \mathcal F_{t-1}=f, a_t=a, \nu_{-\ell}=x)
&\le (1-p)p\KL(P_+ \| P_-)
     + p(1-p)\KL(P_- \| P_+) \\
&\le p(\KL(P_+ \| P_-)
           + \KL(P_- \| P_+)) = 4\Delta^2p.
\end{align*}
For any \(t\in\mathcal I_\ell\), taking expectation over $(\mathcal F_{t-1},a_t,\nu_{-\ell})$ yields
\begin{align*}
I(\nu_\ell;r_t \mid \mathcal F_{t-1}, a_t, \nu_{-\ell})
&\le  4\Delta^2 
   \mathbb{E}\Bigl[\mathds 1\{\widetilde \nu_\ell \le t ,a_t=2 \} + \mathds 1\{\widetilde \nu_\ell > t ,a_t=1 \}\Bigr].
\end{align*}
Finally, for any $\ell \in [L/2]$, 
\begin{equation}\label{eq:fanos-always-satisficing}
\sum_{t=1}^T I\bigl(\nu_\ell;r_t\mid \mathcal F_{t-1},a_t,\nu_{-\ell}\bigr) \leq 4\Delta^2\mathbb{E}\Bigl[\sum_{t=(\ell-1)m+1}^{\widetilde\nu_\ell-1} \mathds 1\{a_t=1\}+\sum_{t=\widetilde\nu_\ell}^{\ell m} \mathds 1\{a_t=2\}\Bigr].
\end{equation}

\noindent\textbf{Construct $\widehat \nu_\ell(Z,\nu_{-\ell})$ with upper bounded identification error.} Fix \(\ell\in[L/2]\). For each candidate change-point \(\tau\in\mathcal T\), define $B_{\ell,\tau}
:= \bigl[(\ell-1)m+\tau,\;(\ell-1)m+\tau+l-1\bigr]$ and $N_{\ell,\tau}(Z)
:= \sum_{t\in B_{\ell,\tau}} \mathds 1\{a_t=1\}.$ Fix an integer $c$ with $1\le c\le l$. For each $\tau$ such that $N_{\ell,\tau}(Z)\ge c$,
let $\overline X_{\ell,\tau}(c)$ denote the average of the \emph{first $c$} rewards
of arm $1$ observed within $B_{\ell,\tau}$. For each $\tau$ such that $N_{\ell,\tau}(Z)<c$, let $\overline X'_{\ell,\tau}(l-c)$ denote
the average of the \emph{first $l-c$} rewards of arm $2$ observed within $B_{\ell,\tau}$.
Define the estimators
\begin{align*}
\widehat \nu_\ell'(Z,\nu_{-\ell})
 &:=
\min\Bigl\{\tau\in\mathcal T:\ N_{\ell,\tau}(Z) \ge c
\ \text{ and }\ \overline X_{\ell,\tau}(c)\ge S\Bigr\}\\
\widehat \nu_\ell''(Z,\nu_{-\ell})
 &:=
\min\Bigl\{\tau\in\mathcal T:\ N_{\ell,\tau}(Z) <c
\ \text{ and }\ \overline X_{\ell,\tau}'(l-c)\leq S\Bigr\}
\end{align*}
We next state a lemma that bounds the error probabilities of the two estimators above; its proof is deferred to Section~\ref{sec:proof-two-estimators}.
\begin{lemma}\label{lem:two-estimators}
Fix integers 
$l,c\in\mathbb N$ with $1\le c\le l$, and $\ell \in [L/2]$. Then,
\begin{align*}
\mathbb P(\widehat \nu_\ell'(Z,\nu_{-\ell})\neq \nu_\ell)
 &\le
\mathbb P(N_{\ell,\nu_{\ell}}(Z)< c) + n\exp(-c\Delta^2/2),\\
\mathbb P(\widehat \nu_\ell''(Z,\nu_{-\ell})\neq \nu_\ell)
 &\le
\mathbb P\big(N_{\ell,\nu_{\ell}}(Z)\geq c\big) + n\exp(-(l-c)\Delta^2/2).
\end{align*}
\end{lemma}
Construct the final estimator by flipping an independent coin
$\Xi\sim \Unif\{0,1\}$. If $\Xi=0$, set
$\widehat \nu_\ell(Z,\nu_{-\ell}) := \widehat \nu_\ell'(Z,\nu_{-\ell})$; otherwise, set
$\widehat \nu_\ell(Z,\nu_{-\ell}) := \widehat \nu_\ell''(Z,\nu_{-\ell})$.
Combining Lemma~\ref{lem:two-estimators} with the choices
$c=\lceil l/2\rceil$, $l=\big\lceil 6\log(n)/\Delta^2\big\rceil+1$, and $n\ge 2$, we obtain
\begin{equation}\label{eq:perr-upper}
\mathbb P(\widehat \nu_\ell(Z,\nu_{-\ell})\neq \nu_\ell)
\ \le 1/2+  n\exp(-(l-1)\Delta^2/2)
\ \le 1/2+ 1/n^2\leq 3/4,
\end{equation}
Combining \eqref{eq:fanos-always-satisficing} and \eqref{eq:perr-upper}, and using~\eqref{eq:fano1-bandit1}, 
we obtain
\[
\E[R_T^S(\pi; \nu)] = \Delta\cdot \sum_{\ell\in [L/2]}\mathbb{E}\Bigl[\sum_{t=(\ell-1)m+1}^{\widetilde\nu_\ell-1} \mathds 1\{a_t=1\}+\sum_{t=\widetilde\nu_\ell}^{\ell m} \mathds 1\{a_t=2\}\Bigr]
 \ge
\frac{L}{8\Delta}\Bigl(\frac{1}{4}\log n - \log 2\Bigr).
\]

\noindent\textbf{Optimize $n$.} Recall the constraints $nl\le 2T/L-1$, $l=\big\lceil 6\log(n)/\Delta^2\big\rceil+1$, and $n\ge 2$.
Ignoring the ceiling for simplicity, it suffices to choose $n\ge 2$ such that $n\log n \;\le\; \frac{\Delta^2}{6}(\frac{2T}{L}-1).$  Since $\Delta\le \tfrac12$ and $T\Delta^2/L \ge 13$, we further have $\frac{\Delta^2}{6}(\frac{2T}{L}-1)\ge 13/3-1/32 > 4.$ By Lemma~\ref{lem:W}, we may choose $n\ge 2$ so that
$\E\!\left[R_T^S(\pi;\nu)\right]\gtrsim \frac{L}{\Delta}\log(\Delta^2T/L)$.
This completes the proof of Theorem~\ref{thm:lower-general-1}.
\subsubsection{Proof of Lemma \ref{lem:two-estimators}}\label{sec:proof-two-estimators} In this section, we prove the upper bound on $\mathbb P(\widehat \nu_\ell'(Z,\nu_{-\ell})\neq \nu_\ell)$ in Lemma~\ref{lem:two-estimators}; the second bound follows by an analogous argument. Let \(E:=\{N_{\ell,\nu_\ell}(Z)<c\}\). Then
\[
\mathbb P(\widehat \nu_\ell'(Z,\nu_{-\ell})\neq \nu_\ell)
\le
\mathbb P(E)
+\mathbb P(\widehat \nu_\ell'(Z,\nu_{-\ell})\neq \nu_\ell, E^{\mathrm c}).
\]
On the event \(E^{\mathrm c}\), the estimator $\widehat \nu_\ell'(Z,\nu_{-\ell})$ can make error only if either the block at the true change point is not declared, i.e.,
\(\overline X_{\ell,\nu_\ell}<S\); or some earlier candidate is declared first,
i.e., \(\overline X_{\ell,\tau}\geq S\) for some \(\tau<\nu_\ell\) with
\(N_{\ell,\tau}(Z)\ge c\).

For the first event, since \(\overline X_{\ell,\nu_\ell}\sim \mathcal N(S+\Delta,1/c)\)
on \(E^{\mathrm c}\), the Gaussian tail bound gives $\mathbb P(\overline X_{\ell,\nu_\ell}<S)
=
\mathbb P\big(\overline X_{\ell,\nu_\ell}-(S+\Delta)<-\Delta\big)
\le \exp(-c\Delta^2/2).$

Next, fix any \(\tau<\nu_\ell\). On \(E^{\mathrm c}\cap\{N_{\ell,\tau}(Z)\ge c\}\),
we have \(\overline X_{\ell,\tau}\sim \mathcal N(S-\Delta,1/c)\), and hence $\mathbb P\big(\overline X_{\ell,\tau}\geq  S\big)
=
\mathbb P\big(\overline X_{\ell,\tau}-(S-\Delta)\ge \Delta\big)
\le \exp(-c\Delta^2/2).$

Taking a union bound over the at most \(n-1\) indices \(\tau\in\mathcal T\) with
\(\tau<\nu_\ell\), we obtain
\[
\mathbb P(\widehat \nu_\ell'(Z,\nu_{-\ell})\neq \nu_\ell, E^{\mathrm c})
\le (n-1)\exp(-c\Delta^2/2)
+\exp(-c\Delta^2/2)
= n\exp(-c\Delta^2/2).
\]
\subsection{Proof of Theorem \ref{thm:upper-general}}\label{sec:proof-upper-general}
By the definition of the satisficing regret in \eqref{eq:satisficing regret}, we have
\begin{equation}\label{eq:regret-representation}
R_T^S(\pi;\nu)
=
\sum_{\ell=0}^{L-1} \sum_{t=T_\ell}^{T_{\ell+1}-1}
\sum_{a=1}^K \Delta_{a,\ell}\mathbb{P}_{\pi,\nu}(a_t = a)
=
\sum_{\ell=0}^{L-1} \sum_{t=T_\ell}^{T_{\ell+1}-1}
\sum_{a:\mu_{a,\ell} <S}
\Delta_{a,\ell}\mathbb{P}(a_t = a),
\end{equation}
where, in the last equality, we simplify the notation by omitting the dependence on the policy $\pi$ and the environment $\nu$. We first introduce some useful notation. For $\Delta>0$, define
\begin{equation*}
n_\Delta
 := 
\min\Bigl\{n\in[T] : \sqrt{\frac{2(4\log T+\log K)}{n}} \le \Delta/2\Bigr\}
 \le 
c_1 \frac{\log TK}{\Delta^2},
\end{equation*}
where \(c_1>0\) is a universal constant. Consequently, we have
\(\beta(n) \le \Delta/2\) for all \(n\) such that \(n_\Delta\le n \le T\).  Define the event $\mathcal E_{\operatorname{win}}$ that for every arm $a$, epoch $\kappa$,
time $t\le T$, and every dyadic window $w\in\mathrm{WinSet}(\mathcal B_a^{(\kappa)}(t))$
whose last $w$ samples lie entirely inside a stationary segment $\ell$,
we have $\Big|\mathrm{AvgLast}\big(\mathcal B_a^{(\kappa)}(t),w\big)-\mu_{a,\ell}\Big|
\ \le\
\beta(w).$ By Hoeffding's inequality for $1$-sub-Gaussian variables, for fixed $(a,\kappa,t,w)$
whose last $w$ samples lie in some stationary segment $\ell$,
\begin{align*}
\mathbb P\left(\Big|\mathrm{AvgLast}\big(\mathcal B_a^{(\kappa)}(t),w\big)-\mu_{a,\ell}\Big|
\geq  
\beta(w)\right)&\le 
2\exp\big(-\frac{w}{2}\beta(w)^2\big)=
2\exp(-4\log T-\log K)=
\frac{2}{KT^4}.
\end{align*}
Taking a union bound over all $a\in[K]$, epochs $\kappa\le T$, times $t\le T$,
and windows $w\in\mathrm{WinSet}(\mathcal B_a^{(\kappa)}(t))$, and using that
$|\mathrm{WinSet}(\mathcal B_a^{(\kappa)}(t))|\le T$, we obtain $\mathbb P(\mathcal E_{\operatorname{win}}^c) \le \frac{2}{T}.$ On the event $\mathcal E_{\operatorname{win}}$, the windowed confidence bounds behave as follows:
\begin{itemize}
  \item If \(\mu_{a,\ell}=S-\Delta\) and \(\mathcal B_a^{(\kappa)}\) contains a suffix window of length \(w\ge 2n_\Delta\) entirely within segment \(\ell\), then \(\mathrm{UCB}^{\mathrm{win}}(\mathcal B_a^{(\kappa)})\le S\). Hence any leader whose mean stays \(\le S-\Delta\) in an epoch is demoted after at most \(2n_\Delta\) pulls within a single segment.
  \item  Similarly, any arm whose mean stays \(\ge S+\Delta\) in an epoch is certified as satisficing once it accumulates \(2n_\Delta\) pulls within a single segment.
  \item Within a fixed segment \(\ell\), any non-satisficing arm can be certified as leader at most once.
  \item Within a fixed segment \(\ell\), once an optimal (satisficing) arm is certified as leader, it is never demoted before the segment ends.
\end{itemize}
On the event $\mathcal E_{\operatorname{win}}$, the exploration length on each segment $\ell$ is
at most $2K n_{\Delta_{\ell}^*}$: during exploration the algorithm runs a
fixed round–robin schedule, and once the optimal arm on segment $\ell$ has
been pulled $2n_{\Delta_{\ell}^*}$ times, it is certified as satisficing
and exploration terminates. Moreover, since every non-satisficing arm
$a$ can be certified as leader at most once in segment $\ell$ and, once
certified, is demoted after at most $2n_{\Delta_{a,\ell}}$ pulls, the
regret incurred on segment $\ell$ is bounded by $\sum_{a:\mu_{a,\ell}<S}(\frac{1}{\Delta_{a,\ell}} + \frac{\Delta_{a,\ell}}{(\Delta_{\ell}^*)^2}\big)\cdot \log KT \lesssim \sum_{a:\mu_{a,\ell}<S}(\frac{1}{\Delta_{a,\ell}} + \frac{\Delta_{a,\ell}}{(\Delta_{\ell}^*)^2}\big)\cdot \log  T,$ up to a universal constant. Therefore, we have
\begin{align*}
R_T^S(\pi;\nu) &= \mathbb{E}_{\pi,\nu}\Big[ \sum_{t=1}^T (S-\mu_t(a_t))_+ \mid \mathcal E_{\operatorname{win}}\Big]\mathbb P(\mathcal E_{\operatorname{win}}) + \mathbb{E}_{\pi,\nu}\Big[ \sum_{t=1}^T (S-\mu_t(a_t))_+ \mid \mathcal E_{\operatorname{win}}^c\Big]\mathbb P(\mathcal E_{\operatorname{win}}^c)\\
&\lesssim \sum_{\ell=0}^{L-1}\sum_{a:\mu_{a,\ell}<S}(\frac{1}{\Delta_{a,\ell}} + \frac{\Delta_{a,\ell}}{(\Delta_{\ell}^*)^2}\big)\cdot \log  T + T\cdot \frac{2}{T} \lesssim \sum_{\ell=0}^{L-1}\sum_{a:\mu_{a,\ell}<S}(\frac{1}{\Delta_{a,\ell}} + \frac{\Delta_{a,\ell}}{(\Delta_{\ell}^*)^2}\big)\cdot \log  T,
\end{align*}
thereby completing the proof of Theorem \ref{thm:upper-general}.
\subsection{Proof of Theorem \ref{thm:upper-nodown}}\label{sec:proof-upper-nodown}
In this section, we prove Theorem~\ref{thm:upper-nodown} by
adapting the proof of Theorem~1 in~\cite{michel2022regret} from the
stationary  setting to our nonstationary setting under
Assumptions~\ref{ass:always-realizable} and~\ref{ass:nodown}. We work with the same regret representation as in equation~\eqref{eq:regret-representation}: $R_T^S(\pi;\nu)
=
\sum_{\ell=0}^{L-1} \sum_{t=T_\ell}^{T_{\ell+1}-1}
\sum_{a:\mu_{a,\ell} <S}
\Delta_{a,\ell}\mathbb{P}(a_t = a).$  Following the event decomposition used in Appendix~A of~\cite{michel2022regret},
we write, for each arm $a$ and time $t$, $\left\{a_t=a\right\} \subset\{t=a\} \cup\left\{a_t=a, Z_t^c\right\} \cup\left\{a_t=a, Z_t\right\},$ where $Z_t:=\left\{\forall a \in [K], \hat{\mu}_t(a)<S\right\}$. For the first event, we have $\sum_{\ell=0}^{L-1} \sum_{t=T_\ell}^{T_{\ell+1}-1}
\sum_{a:\mu_{a,\ell} < S}
\Delta_{a,\ell}\mathds{1}\{t = a\} \leq \sum_{a:\Delta_{a,\min}>0}\Delta_{a,\max}.$

For each arm $a \in [K] \setminus \{a^*\}$, define $T(a): =\max\{t \in [T]: \mu_t(a)<S\}$. For $n \le T(a)$, let $\hat\mu_{a,n}$ denote the empirical mean of arm $a$
based on the first $n$ pulls of $a$ that occur strictly before time $T(a)$. For the second event, Assumption~\ref{ass:nodown} implies that
\begin{align*}
\sum_{\ell=0}^{L-1} \sum_{t=T_\ell}^{T_{\ell+1}-1}
\sum_{a:\mu_{a,\ell} < S}
\Delta_{a,\ell}\mathbb{P}(a_t=a, Z_t^c) &\leq \sum_{a:\Delta_{a,\min}>0}
\Delta_{a,\max}\sum_{n=1}^{T(a)}\mathbb{P}(\hat\mu_{a,n}(a) \geq S)\\
&\leq \sum_{a:\Delta_{a,\min}>0}
\Delta_{a,\max}\sum_{n=1}^{T(a)}\exp(-\frac{n\Delta_{a,\min}^2}{2})\lesssim \sum_{a:\Delta_{a,\min}>0} \frac{\Delta_{a,\max}}{\Delta_{a,\min}^2}.
\end{align*}
Let $\hat\mu_{a^*,n}$ denote the empirical mean of arm $a^*$
based on the first $n$ pulls of $a^*$. For the last event, a similar argument, together with Assumption~\ref{ass:always-realizable}, implies that
\begin{align*}
\sum_{\ell=0}^{L-1} \sum_{t=T_\ell}^{T_{\ell+1}-1}
\sum_{a:\mu_{a,\ell} < S}
\Delta_{a,\ell}\mathbb{P}(a_t=a, Z_t)&= \sum_{\ell=1}^L \sum_{t=T_\ell}^{T_{\ell+1}-1}
\sum_{a:\mu_{a,\ell} < S}
\Delta_{a,\ell}\mathbb{P}(a_t=a^*, Z_t)\lesssim \sum_{a:\Delta_{a,\min}>0} \frac{\Delta_{a,\max}}{\Delta_{*}^2}.
\end{align*}
Combining the three events, we complete the proof of Theorem \ref{thm:upper-nodown}.

\section{Conclusion}
This paper studies satisficing regret in nonstationary bandits. Our results reveal a sharp contrast with the stationary realizable setting: under nonstationarity, even a single change in the environment is enough to make \(T\)-independent satisficing regret unattainable in general. In particular, in the realizable
piecewise-stationary setting with \(L\) stationary segments, we show that the
  optimal regret scales as \(\Theta(L\log T)\) whenever
\(L\ge 2\). Our lower bounds are based on a new post-interaction reference framework, together with a conditional Fano argument, that strictly extends the interactive Fano methods from stationary bandits to the nonstationary setting.

We also show that \(T\)-independent satisficing regret remains achievable under
additional structural assumptions. Important directions for future work include characterizing
 necessary and sufficient conditions for such constant regret under
nonstationarity and studying whether one can design a single algorithm that adaptively
achieves the optimal rate in both the general and special regimes.
\ACKNOWLEDGMENT{Q. Xie and Y. Zhang are supported in part by National Science Foundation (NSF) grants CNS-1955997
and EPCN-2339794 and EPCN-2432546.}


\bibliography{main} 

\bibliographystyle{informs2014} 




\end{document}